\documentclass[10pt,twocolumn,letterpaper]{article}

\usepackage{cvpr}              
\usepackage[most]{tcolorbox}
\usepackage{cprotect}
\tcbuselibrary{skins,breakable,theorems}
\usepackage{amsmath, amssymb, amsfonts, amsthm, mathtools, bm, bbm}
\usepackage{enumitem}
\usepackage{subcaption}
\usepackage{algorithm}   
\usepackage{listings}    
\usepackage[table]{xcolor}
\usepackage[noend]{algpseudocode}

\DeclareMathOperator*{\argmin}{arg\,min}
\newcommand{\code}[1]{\texttt{#1}}

\usepackage{bibunits}                       
\defaultbibliographystyle{ieeenat_fullname}
\defaultbibliography{main}

\newtheorem{definition}{Definition}

\newtheorem{theorem}{Theorem}
\newtheorem{lemma}{Lemma}

\DeclareMathOperator{\supp}{supp}

\DeclarePairedDelimiter{\norm}{\lVert}{\rVert}

\DeclarePairedDelimiterX{\inner}[2]{\langle}{\rangle}{#1,\,#2}

\definecolor{cvprblue}{rgb}{0.21,0.49,0.74}
\usepackage[pagebackref,breaklinks,colorlinks,allcolors=cvprblue]{hyperref}

\def\paperID{18246} 
\def\confName{CVPR}
\def\confYear{2026}

\title{Rethinking Model Selection in VLM Through the Lens \\of Gromov-Wasserstein Distance}

\author{%
Muyang Li$^{1}$ \enspace
Yucheng Liu$^{2}$ \enspace
Jianbo Ma$^{2}$\\
Elliot Osborne$^{2}$ \enspace
Bo Han$^{3}$ \enspace
Tongliang Liu$^{1}$\thanks{Corresponding Author}\\
$^{1}$Sydney AI Centre, The University of Sydney\\
$^{2}$Dolby Laboratories \quad $^{3}$TMLR Group, Hong Kong Baptist University\\
{\ttfamily\small \{muyang.li,tongliang.liu\}@sydney.edu.au}\\
{\ttfamily\small \{yucheng.liu,jianbo.ma,elliot.osborne\}@dolby.com}\\
{\ttfamily\small bhanml@comp.hkbu.edu.hk}%
}

\begin{document}
\maketitle
\begin{abstract}
Vision-Language Models (VLMs) have enhanced traditional LLMs with visual capabilities through the integration of vision encoders. While recent works have explored various combinations of vision encoders and LLMs, there still lacks a principled understanding of what makes a vision encoder suitable for VLM alignment. In this paper, we systematically investigate this question via comprehensive experiments on a curated collection of 18 pre-trained vision encoders from diverse sources. We first demonstrate that common practices, such as choosing encoders with the largest size or highest zero-shot accuracy, consistently fail to identify optimal models. In fact, these metrics show only weak to moderate correlation with VLM performance. This intriguing finding begs a fundamental question: \textit{What factors of vision-encoders matter in VLM?} Through comprehensive analysis, we identify that the \textit{\textbf{structural similarity}} across modalities plays a crucial but previously overlooked role in vision-encoder selection, which we measure using the Gromov-Wasserstein distance as a proxy. From a theoretical perspective, we show that the learnability of cross-modality mapping can be \textbf{provably} associated with the Gromov-Wasserstein distance. Empirical verification on 60+ full VLM training runs shows that our proposed inference-only metric performs significantly better than alternative model selection strategies and exhibits a much stronger correlation with final VLM performance, thereby enabling efficient and effective prediction of VLM performance before full training.
\end{abstract}
\vspace{-7mm}
\section{Introduction}
\label{sec:intro}

Vision-language models (VLMs) have become an indispensable component of modern AI systems, ranging from state-of-the-art proprietary AI systems such as ChatGPT-5, to open-sourced popular models such as LLaVa~\citep{liu2023visual}, Cambrian~\citep{tong2024cambrian}, etc. The dominant paradigm for current state-of-the-art (SOTA) VLMs ~\citep{bai2025qwen2,tong2024cambrian,deitke2025molmo,li2024llavaonevision,xue2024xgen,zheng2025chainoffocus} can be depicted by the seminal work of Visual Instruction Tuning~\citep{liu2023visual}, which is typically characterized as a ``pre-training then fine-tuning" process: a projector is trained to semantically align the representation spaces of the vision encoder and LLM, and subsequently, the model undergoes joint full-parameter fine-tuning of both towers for a cohesive multi-modal understanding and instruction-following capabilities.

Despite the widespread success and the adoption of this simple paradigm, many questions remain unclear in this process. Specifically, prior studies have explored numerous combinations of LLM and vision-encoders~\citep{tong2024cambrian,cocchi2025llava}, and a mismatch between the model size and performance after fine-tuning has been observed~\citep{li2024bigger}. As a foundational step, this paper presents a holistic evaluation of 18 state-of-the-art (SOTA) vision encoders from diverse sources, by performing complete visual instruction tuning~\cite{liu2024improved} loop on all of them. We show that traditional wisdom such as selecting the model with the largest size or highest zero-shot accuracy consistently fails to identify best choice of vision encoders. Moreover, our correlation analysis suggests that there is even no statistically meaningful correlation between the capability of vision encoders and final VLM performances. This intriguing observation begs the question of: 
\begin{tcolorbox}[enhanced,attach boxed title to top center={yshift=-0mm,yshifttext=-1mm},
colback=gray!5!white,colframe=gray!75!black,colbacktitle=red!80!black,
  title=,fonttitle=\bfseries,
  boxed title style={size=small,colframe=red!50!black} ]
  \emph{\quad What makes a good vision encoder for VLM?}
\end{tcolorbox}
\vspace{-1mm}

\begin{figure*}[t]                 
  \centering
  \begin{subfigure}[t]{0.33\textwidth}
  \includegraphics[width=\linewidth]{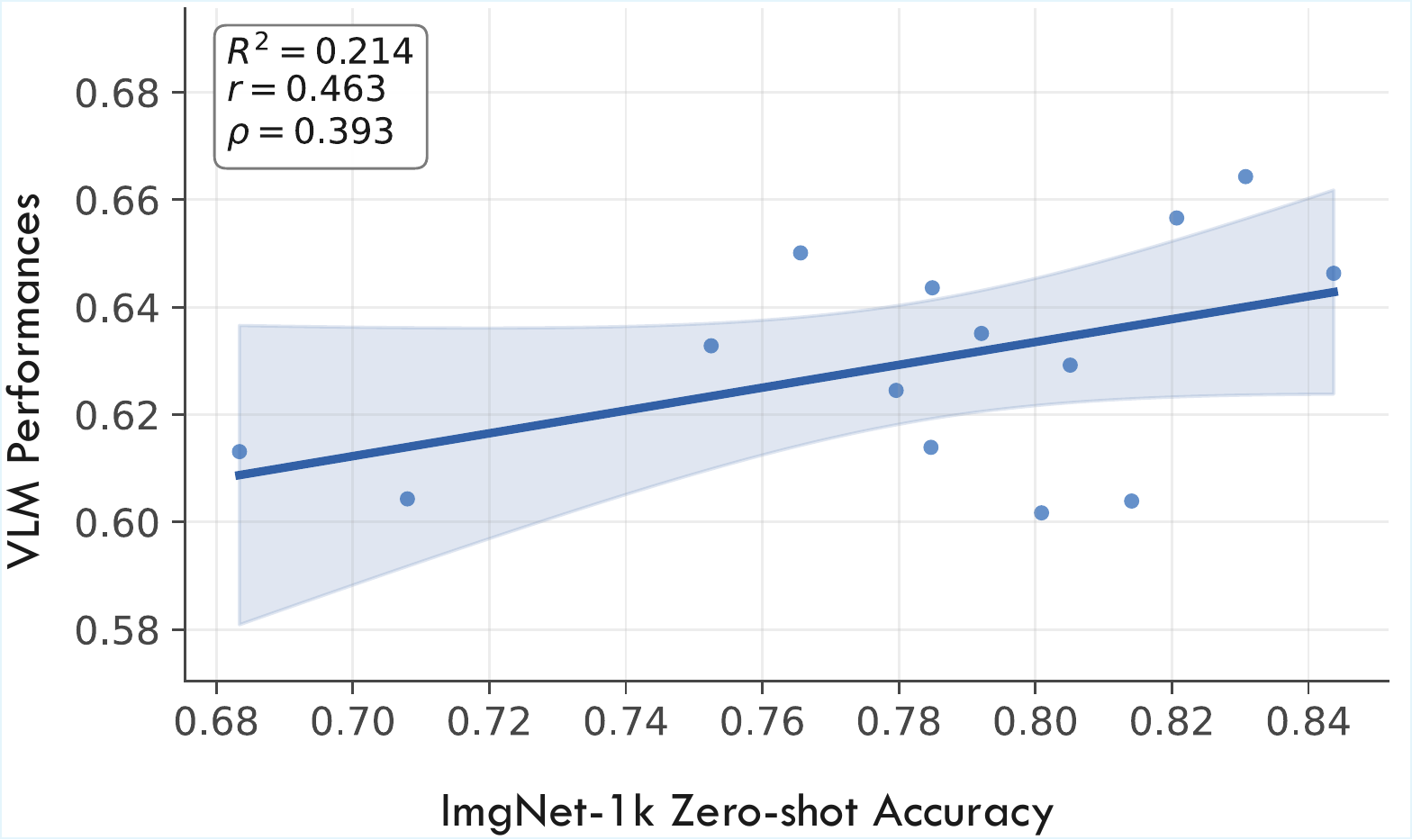}
  \end{subfigure}
  \begin{subfigure}[t]{0.33\textwidth}
  \includegraphics[width=\linewidth]{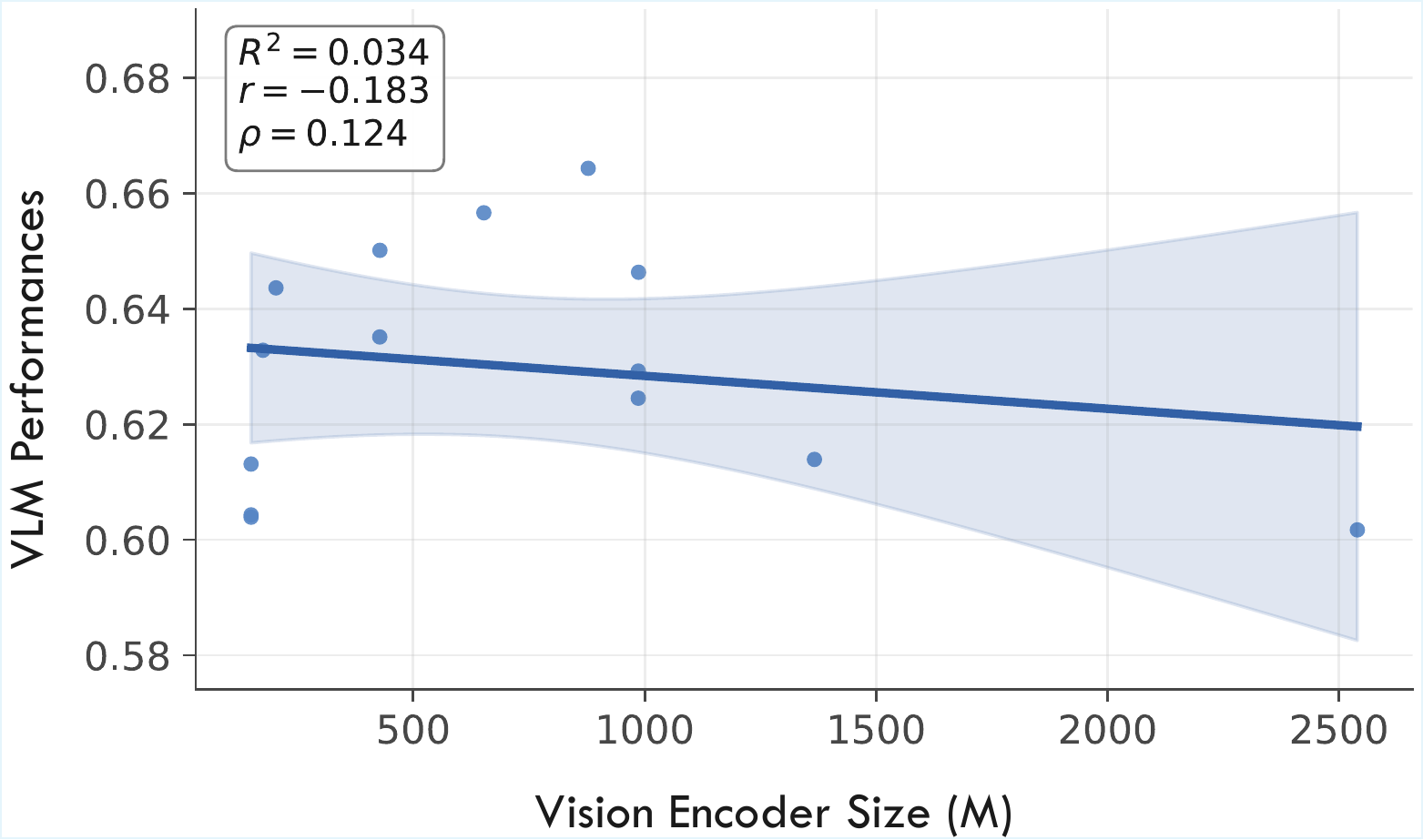}
  \end{subfigure}
  \begin{subfigure}[t]{0.33\textwidth}
  \includegraphics[width=\linewidth]{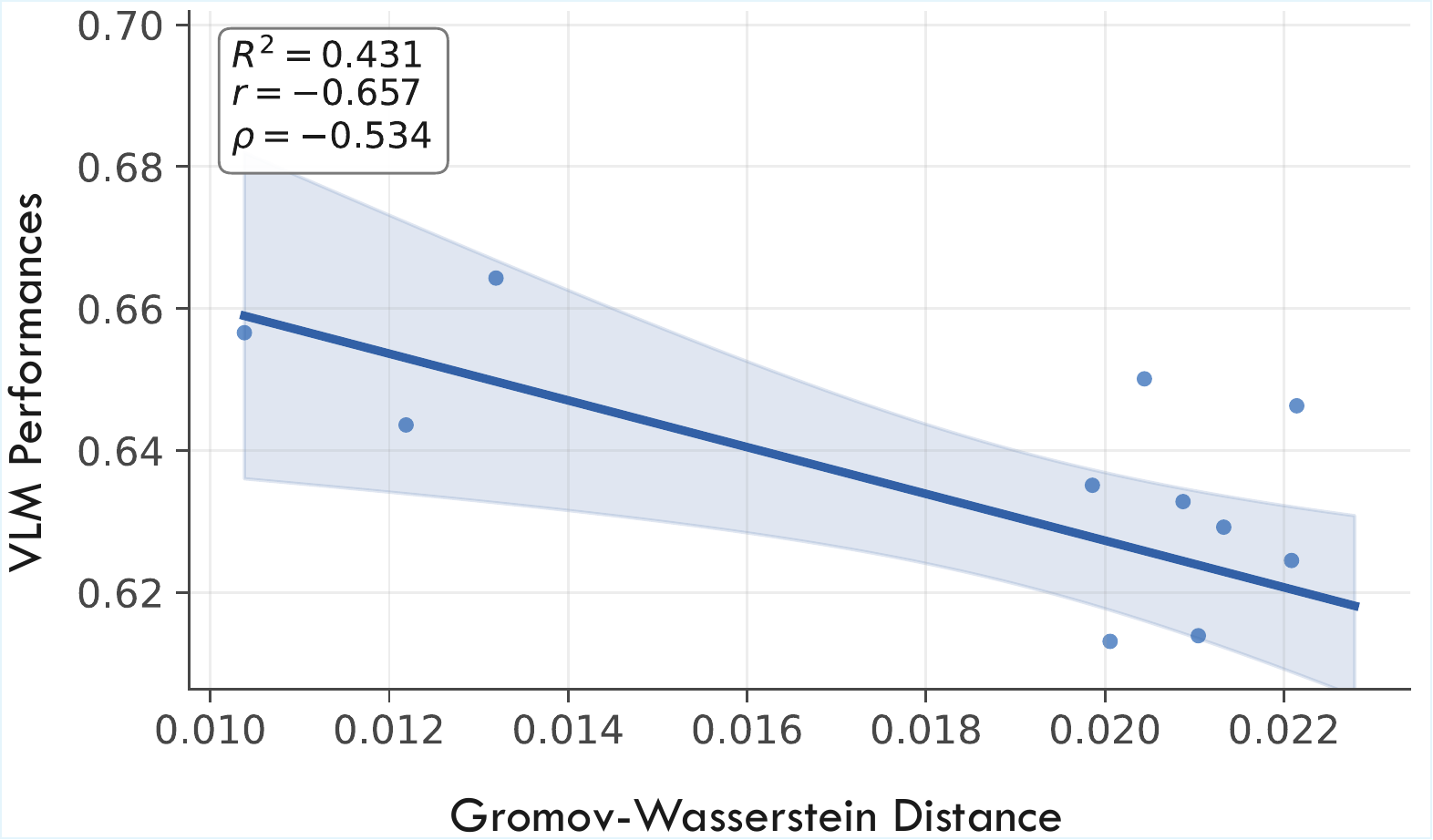}
  \end{subfigure}
  \label{fig:motivation}
  \vspace{-5mm}
  \caption{From left to right, we can see the correlation analysis of zero-shot classification accuracy, vision encoder size, and GW distance against the final VLM performances, respectively. $r$: Pearson correlation, $\rho$: Spearman correlation, R$^2$: R-squared coefficient. As we can see, the correlation between GW distance and final VLM performances is notably higher than naive baselines.}
  \vspace{-3mm}
\end{figure*}

In this paper, we aim to gain a holistic understanding to this question: if the performances of pre-trained vision encoders cannot explain the performances of VLM after visual instruction tuning, what could be the factors that are actually crucial? We begin with a simple intuition: beyond raw visual capability, there exists a notion of \textbf{\emph{compatibility}} between the vision encoder and the LLM. Depending on the pre-training data, model architecture, and training strategy, the induced visual representations can be more or less heterogeneous with respect to the LLM’s representation space, directly affecting the difficulty of aligning them. We hypothesize that a key aspect of this compatibility is the \textbf{\emph{structural similarity}} between visual and textual representations, which we formalize as an appropriate similarity measure between their induced distributions. When a vision encoder produces representations that are structurally incompatible with those of the LLM (i.e., the distance between the visual and textual representation distributions is large), the resulting VLM may perform poorly after fine-tuning, even if the encoder itself exhibits strong standalone visual capability.

However, one challenge emerging from the attempt to compute distributional discrepancy between different modalities is that: distribution of different modality representations resides in different metric-measure space (mm-space). That is: (i) their measures are different, since the measurable function that generates those representations are different, in other words, naive distance of representations from different modalities have no canonical meaning; (ii) their metric function are different, for example, representations from different modalities have different dimension, which means that no unified metric function can naively compute their distance. Motivated by this, we seek the notion of Gromov-Wasserstein (GW) distance~\cite{memoli2011gromov} as a proxy to measure the distance between distributions from different mm-spaces (modalities). Instead of calculating the inter-distribution distance, GW distance measures how similar two mm-spaces are by finding a correspondence between their points that minimizes the distortion of pairwise distances. In other words, it compares structural similarity up to isometry without needing a common ambient space.

We summarize our contributions as follows: (i) we systematically study the problem of pre-trained vision encoder selection in VLM training, unveiling the fact that there currently lacks a plausible explanatory variable to correlate the factors of vision encoder to final VLM performances; (ii) we unveil the crucial factor in the success of VLM - the distributional compatibility between vision and text representations, which we instantiate
 through GW distance; (iii) we prove that the upper-bound of Lipschitz constant of the Bayes optimal hypothesis of cross-modality projector is associated with the $\infty$-norm GW distance, justifying that, if vision and text representations are of small GW distance, then the optimal mapping between them can essentially be learned with simpler hypothesis; (iv) through comprehensive experimental verifications with 60+ full VLM training runs, our proposed metrics show consistent performance gains compared to baselines, and showing strong Pearson correlation to the final VLM performances, whereas baselines' correlations are only weak to moderate.
\vspace{-3mm}
\section{Related Works}

\textbf{Gromov-Wasserstein Distance.} The Gromov-Wasserstein (GW) distance~\cite{memoli2011gromov} was proposed to quantify the similarity between different spaces, with its application in Graph Matching~\cite{xu2019gromov}, Flow Matching Diffusion~\cite{klein2024genot}, Saliency Detection~\cite{zhang2021deepacg}, etc. Essentially, in areas where we wish to measure the similarity across heterogeneous domains, GW distance could be useful since it is invariant to choice of metric function or measure equipped with that domain. Subsequent methodological improvements over vanilla GW distance including Entropic Gromov-Wasserstein distance~\cite{peyre2016gromov}, which applied entropic penalization to transport plan for more smoothness, etc.

\noindent\textbf{The Platonic Representation Hypothesis.} A notable existing work relevant to our study is the Platonic Representation Hypothesis~\cite{huh2024platonic}. This hypothesis stipulates that as the size of vision encoders and LLMs scales, their representations across modalities effectively converge. This implies that, despite being trained with significantly different objective functions, data, and hypothesis classes, different models eventually acquire similar understandings of concepts located in different modalities. In fact, \cite{huh2024platonic} proposed a metric for estimating this alignment score called Mutual Nearest Neighbors (MutualNN). Conceptually, MutualNN is similar to the Gromov-Wasserstein (GW) distance in that both techniques avoid estimating cross-modality representations directly. Instead, both compute within-modality distances first and then compare these intra-space metrics across domains. Specifically, MutualNN estimates the overlap of the local neighborhoods, whereas GW first solves for an optimal coupling between the given representations and then calculates the total discrepancy between the intra-domain pairwise distances weighted by this optimal transport plan.
A more comprehensive discussion of related works can be found in the Appendix~\ref{app:relate}.

\section{Gromov-Wasserstein for Model Selection}
\label{sec:method}
In this section, we present our main methodology - using Gromov-Wasserstein (GW) distance as a proxy to measure the structural similarity across modalities. Our intuition is that, if two mm-spaces are more structurally similar, then the concepts encoded in respective modalities are easier to be reconciled into a unified space, and the resulting VLM will be more capable of jointly understanding the multi-modal representations. 
\vspace{-5mm}
\paragraph{Preliminary.} We begin with a formulation of our problem at hand, considering a collection of candidate vision encoder $\mathcal{V}=\{V_{1},\cdots,V_{m}\}$, and a target LLM $\mathcal{F}$. The objective is to determine which vision encoder is likely to produce the best VLM with the given LLM, without actually fine-tuning every possible combination of candidate vision encoders and the LLM, since the computational cost for doing so will be forbidden. Formally, let $\bm{X}=\{\bm{x_i}\}_{i=1}^{n}\in\mathbb{R}^{k\times n}$ be the collection of encoded image representations, and $\bm{Y}=\{\bm{y_i}\}_{i=1}^{n}\in\mathbb{R}^{h\times n}$ be the collection of encoded text representations by LLM. 
Note that we have a pool of different vision encoders, each generating a unique $\bm{X}$. 
Moreover, define metric-measure-spaces (mm-spaces) $\mathcal{X} = (\bm{X}, d_{\mathcal{X}}, \mu_{\mathcal{X}})$ and $\mathcal{Y} = (\bm{Y}, d_{\mathcal{Y}}, \mu_{\mathcal{Y}})$, $d_{\mathcal{X}}, d_{\mathcal{Y}}$ are metric functions, and $\mu_{\mathcal{X}}, \mu_{\mathcal{Y}}$ are Borel measures on $\bm{X}$ and $\bm{Y}$, respectively. We fix $d_{\mathcal{X}}, d_{\mathcal{Y}}$ as angular distance, s.t. $d_{\mathcal{X}}(\bm{x},\bm{x}') =\cos^{-1}\left(\frac{\bm{x}\cdot \bm{x}'}{\Vert\bm{x}\Vert\cdot\Vert\bm{x}'\Vert}\right), \bm{x}, \bm{x}' \in \bm{X}$ and $d_{\mathcal{Y}}(\bm{y},\bm{y}') =\cos^{-1}\left(\frac{\bm{y}\cdot \bm{y}'}{\Vert\bm{y}\Vert\cdot\Vert\bm{y}'\Vert}\right), \bm{y}, \bm{y}' \in \bm{Y}$. Finally, $\pi\in\mathbb{R}^{n\times n}_{+}$ is the transport plan between $\bm{X}$ and $\bm{Y}$.

\begin{figure}[t]                 
  \centering
  \includegraphics[width=\linewidth]{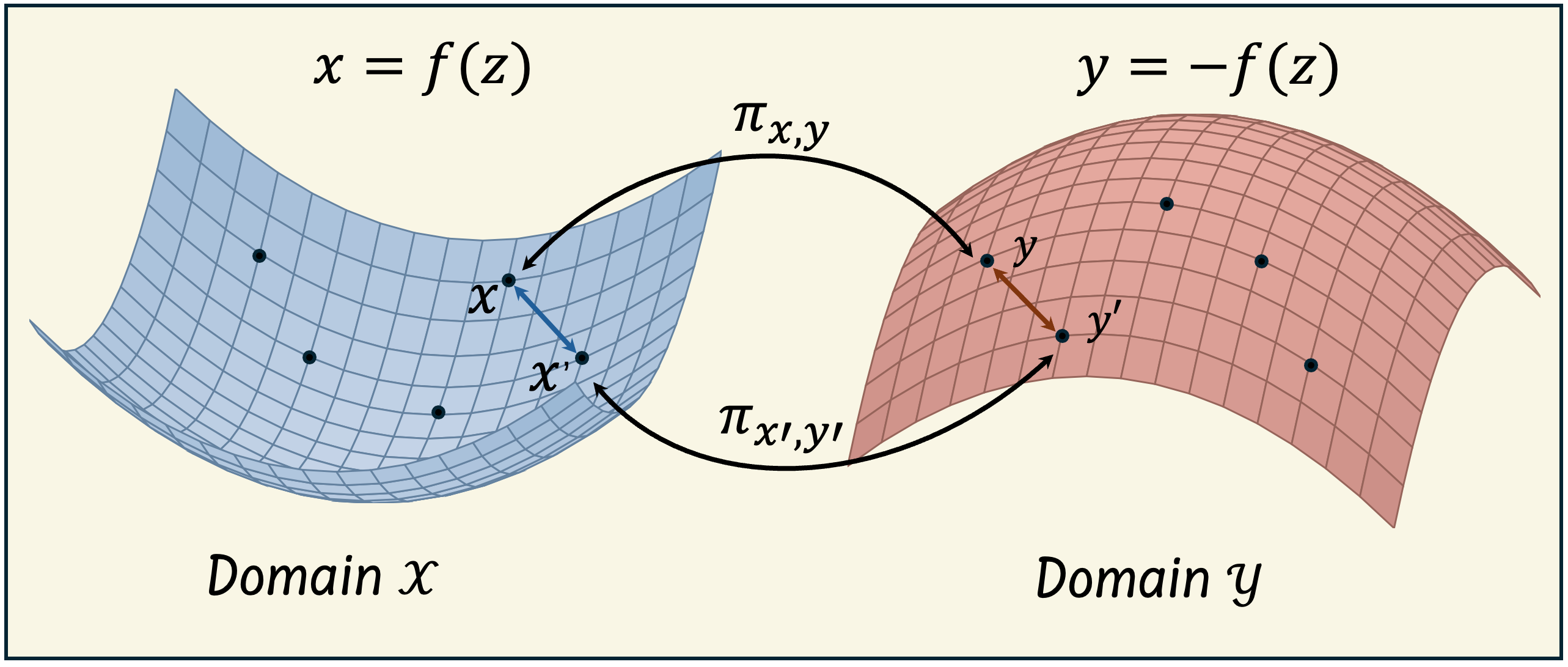}
  \caption{A toy example to show the intuition of GW distance: while it is clear that space $\mathcal{X}$ and $\mathcal{Y}$ are identical up to isometry, direct distance comparison fails to reflect such similarity. However, GW distance can capture such structural similarity by comparing the within-space geometries.}
  \label{fig:GW}
  \vspace{-4mm}
\end{figure}

\paragraph{Gromov-Wasserstein Distance.} As we mentioned above, naively measuring the distance between two different mm-spaces is ill-posed. 
Consider the toy example shown in Figure~\ref{fig:GW}, where we have two different spaces, generated by a function $f(\cdot)$ and its reflection. While it is clear that these two spaces are highly similar, in fact, identical up to isometry, naively computing the absolute distances for matching instances will result in large pairwise distance, falsely suggesting they are actually extremely different. The GW distance, however, can capture such nuanced structural similarity. Instead of comparing the absolute positions of points, GW distance compares their internal relationships. More specifically, given the optimal correspondence between the instances in both spaces, GW distance measures how well the pairwise distances within each space align. In other words, if you pick two points $(\bm{x}, \bm{x}')$ in space $\mathcal{X}$, and their corresponding matched points $(\bm{y}, \bm{y}')$ in $\mathcal{Y}$, GW distance checks if the distance $d_{\mathcal{X}}(\bm{x}, \bm{x}')$ is the same as the distance $d_{\mathcal{Y}}(\bm{y}, \bm{y}')$. It calculates the difference between these internal distances for all pairs, this structural mismatch, for a given correspondence $\pi$, is formally termed the Gromov-Wasserstein discrepancy~\cite{peyre2016gromov}:
\begin{align}\label{equ:distortion}
    \mathcal{E}(\pi) = \sum_{\substack{\bm{x},\bm{x}'\bm{y},\bm{y}'}}\mathcal{L}\left(d_{\mathcal{X}}(\bm{x},\bm{x}'),d_{\mathcal{Y}}(\bm{y},\bm{y}')\right)\pi_{\bm{x},\bm{y}}\pi_{\bm{x}',\bm{y}'},
\end{align}
where $\mathcal{L}$ is the choice of penalty function for pairwise distances. In this paper we will use $\ell$-1 distance for $\mathcal{L}$. Denoting by $\Pi$ the collection of possible couplings, we define the GW distance as:
\begin{align}\label{equ:GW_distance}
    \text{GW} := \inf_{\pi\in\Pi}\mathcal{E}(\pi).
\end{align}

\paragraph{Scale-matching Across Modalities.} The most common choices of penalty function $\mathcal{L}$ are either $\ell$-1 distance or $\ell$-2 distance~\cite{memoli2011gromov,peyre2016gromov,bunne2019learning}, which are sensitive to the inherent differences in unit scale at different domains. To obtain meaningful estimation of GW distance, we first unify the distance metrics in different domains to the same unit scale. Specifically, we use median-ratio matching for normalization, where we fit an affine transformation between the pairwise distance matrix of visual modality, such that the median of its off-diagonal values matches the counterparts in text modality, or vice versa. By doing so, we are making sure the pairwise distance matrices from different modalities are in the same unit scale, while making no harm to the semantic information as the transformation is strictly affine.
\begin{align}\label{equ:median}
    s := \frac{\text{med}(\mathcal{D}_{\bm{Y}})}{\text{med}(\mathcal{D}_{\bm{X}})}, \quad \bar{\mathcal{D}}_{\bm{X}} := s\cdot\mathcal{D}_{\bm{X}},
\end{align}
where $\mathcal{D}_{\bm{X}}, \mathcal{D}_{\bm{Y}}$ are the pairwise distance matrices for $\bm{X}, \bm{Y}$ respectively, and $med$ refers to median function that returns the median of the off-diagonal values of the input matrix. $\bar{\mathcal{D}}_{\bm{X}}$ will be the actual pairwise distance matrix we use when computing the GW distance.

\paragraph{Solving $\pi$ via Optimal Transport.} As we discussed earlier, computing any GW-based distance requires solving for the optimal transport plan across spaces that minimizes the distortion term in~\eqref{equ:distortion}. That is, we need to solve 
\begin{align}
    \pi^* := \argmin_{\pi \in \Pi} \mathcal{E}(\pi).
\end{align}
We follow the optimization procedure of~\cite{peyre2016gromov}, which treats the GW objective as a quadratic functional over the transport polytope and solves it via a conditional-gradient scheme. Starting from an initialized coupling $\pi^{0} \in \Pi$, we iteratively update the transport plan as follows. At iteration $t$, we first compute the gradient of the GW objective with respect to the current coupling:
\begin{align}\label{equ:gradient}
    C^{t} = \nabla_{\pi} \mathcal{E}\big(\pi^{t}\big),
\end{align}
and solve the linear OT subproblem:
\begin{align}
    \tilde{\pi}^{t} 
    \;:=\; 
    \argmin_{\pi \in \Pi} \,\langle C^{t}, \pi \rangle,
\end{align}
which is a standard linear OT problem with cost matrix $C^{t}$. Finally, we update the coupling by:
\begin{align}\label{equ:update:coupling}
    \pi^{t+1} \;=\; (1 - \eta_t)\,\pi^{t} + \eta_t\,\tilde{\pi}^{t},
\end{align}
where $\eta_t \in (0,1]$ is the step size at $t$. Under mild regularity assumptions, this conditional gradient scheme on the compact convex transport polytope converges to a stationary point of the GW objective~\cite{peyre2016gromov}.

\noindent\textbf{Model Selection Algorithm.} Here, we present a unified workflow for using GW distance as a model-selection metric, summarized in Algorithm~\ref{alg:GW}, details the computation of the GW-based compatibility score, enabling us to identify strong VLM backbones in a fully training-free manner.

\begin{algorithm}[h!]
  \caption{GW distance for vision encoder selection}
  \label{alg:GW}
  \begin{algorithmic}[1] 
    
    \Require $\mathcal{V}$, collection of vision encoders; \texttt{LLM}, base LLM for VLM
    
    \State GW\_distances $=$ []
    \For{\texttt{vision\_encoder} in $\mathcal{V}$}
      \State $\bm{X} = \texttt{vision\_encoder}(\text{image})$
      \State $\bm{Y} = \texttt{LLM}(\text{text})$
      \State $\mathcal{D}_{\mathcal{X}} = \texttt{pairwise\_distance}(\bm{X})$
      \State $\mathcal{D}_{\mathcal{Y}} = \texttt{pairwise\_distance}(\bm{Y})$
      \State \Comment{Do median matching normalization in ~\ref{equ:median}}
      \State $\bar{\mathcal{D}}_{\mathcal{X}} = \texttt{median\_matching}(\mathcal{D}_{\mathcal{X}},\mathcal{D}_{\mathcal{Y}})$
      \State \Comment{Compute optimal transport plan $\pi^*$ by Equ.\ref{equ:gradient} to \ref{equ:update:coupling}}
      \State $\pi^* = \texttt{optimal\_transport}(\bar{\mathcal{D}}_{\mathcal{X}},\mathcal{D}_{\mathcal{Y}})$
      \State \Comment{Compute GW distance using Equ.\ref{equ:distortion}}
      \State $\text{GW} = \texttt{GW\_compute}(\pi^*,\bar{\mathcal{D}}_{\mathcal{X}},\mathcal{D}_{\mathcal{Y}})$
      \State $\texttt{append}(\text{GW\_distances}, \text{GW})$
    \EndFor
    \State optimal\_model = $\argmin(\text{GW\_distances})$
    \State \Return optimal\_model
    
  \end{algorithmic}
\end{algorithm}

\section{Theoretical Analysis}

In this section, we provide a theoretical analysis to shed light on the role of structural similarity in cross-modal alignment learning. In particular, we formalize how the Gromov-Wasserstein (GW) distance is related to the learnability of cross-modal mappings. 

\begin{definition}[Bayes-optimal realizability]
    Define the `bayes-optimal realizability' as the cross-modality matching error induced by the bayes-optimal hypothesis $g^*$ over the population $\mathcal{D}$:
    \begin{align}
        \mathcal{R^*}(\mathcal{D}) = \inf_{g} \mathbb{E}_{(\bm{x},\bm{y})\sim\mathcal{D}} d_{\mathcal{Y}}(g(\bm{x}),\bm{y}).
    \end{align}
\end{definition}
Our definition of ``realizability" is a direct adaption of Definition 3 from \cite{lu2023theory}, where the original definition is defined w.r.t. finite-sample and specific choice of pre-defined hypothesis class. Whereas in our case, we are concerned with the realizability of the underlying optimal mapping across modalities, namely $g^*$. The original definition from \cite{lu2023theory} is provided in Appendix \ref{app:theory:def} for completeness.

\begin{definition}[$\infty$-norm GW-distance, Equ. 5.8 \cite{memoli2011gromov}]\label{def:GW_inf}
    \begin{align}
        \text{GW}_{\infty} = \inf_{\pi\in\Pi} \sup_{\substack{(\bm{x},\bm{y}),(\bm{x}',\bm{y}')\\ \in\supp(\pi)}} \left|d_{\mathcal{X}}(\bm{x},\bm{x}') - d_{\mathcal{Y}}(\bm{y},\bm{y}')\right|.
    \end{align}
\end{definition}
By definition, the $\infty$-norm of the GW distance measures the image-text pairs with the maximum gap of distortion under optimal coupling. We are only concerned with the correspondence with non-zero transporting mass (over the support of the transport plan).

\begin{definition}[Worst-case Bayes error]\label{def:bayes}
Defining $\rho^*_{\pi} \in [0,\infty)$ as the worst-case error of the $g^*$, evaluated under the correspondence induced by $\pi$ as
    \begin{align}
        \rho^*_{\pi} := \sup_{(\bm{x},\bm{y})\in\supp(\pi)}d_{\mathcal{Y}}(g^*(\bm{x}),\bm{y}).
    \end{align}
\end{definition}

\begin{theorem}\label{theo:main}
Let $S_{\mathcal{X}}:= \{\bm{x}:(\bm{x}, \bm{y}) \in \supp(\pi^*_\infty)\}$ and $r := \inf\{d_{\mathcal{X}}(\bm{x}, \bm{x}'): \bm{x}, \bm{x}' \in S_{\mathcal{X}},\, \bm{x} \neq \bm{x}'\}$.
Then $g^*$ is $L_{g^*}$-Lipschitz on $S_{\mathcal{X}}$ with:
\begin{align}
    L_{g^*} \leq 1 + r^{-1}(2\rho^{*}_{\pi_{\infty}}+\text{GW}_{\infty}).
\end{align}
\end{theorem}
The proof of Theorem~\ref{theo:main} is deferred to Appendix~\ref{app:theory:proof_theo1}. Theorem~\ref{theo:main} shows that the Lipschitz constant of the Bayes-optimal cross-modal projector $g^*$ is bounded by the GW$_\infty$ distance between modalities. In other words, when the underlying mm-spaces are more compatible in the GW sense, the Bayes-optimal projector can be realized by a hypothesis with smaller Lipschitz constant, i.e., lower functional complexity. This provides a formal justification for the intuition that GW distance is linked to the learnability of cross-modality mappings. Moreover, since PAC-style generalization bounds for modern neural networks depend explicitly on the Lipschitz constant (or spectral norm) of the predictor~\cite{neyshabur2017pac,li2024towards,bartlett2017spectrally}, our bound on $L_{g^*}$ can be directly incorporated into these results to yield GW-dependent generalization guarantees.

\section{Experiments}
\label{sec:exp}
\begin{table*}[t]
\centering
\scriptsize
\setlength\tabcolsep{3pt}
\resizebox{\textwidth}{!}{
\begin{tabular}{l|c|cccccccccc}
\toprule
Method & Average & TextVQA & ScienceQA & GQA & AI2D & MMMU & MME-C & MME-P & SEED-I & POPE & RealWorldQA \\
\midrule
Worst & $60.17$ & $52.35$ & $81.44$ & $57.12$ & $62.66$ & $43.22$ & $297.50$ & $1361.82$ & $62.15$ & $85.68$ & $51.76$ \\
Accuracy & $\underline{64.63}$ & $\underline{57.53}$ & $80.97$ & $\underline{62.99}$ & $64.90$ & $42.33$ & $365.00$ & $\mathbf{1559.15}$ & $\underline{69.29}$ & $\underline{86.76}$ & $\underline{57.91}$ \\
RSA & $62.92$ & $53.64$ & $\underline{81.63}$ & $61.05$ & $\underline{65.06}$ & $\underline{43.67}$ & $335.00$ & $1510.80$ & $66.28$ & $85.58$ & $54.90$ \\
CCA & $61.39$ & $51.43$ & $79.65$ & $59.76$ & $64.80$ & $41.56$ & $\underline{368.21}$ & $1404.15$ & $64.49$ & $84.11$ & $51.90$ \\
MutualNN & $\underline{64.63}$ & $\underline{57.53}$ & $80.97$ & $\underline{62.99}$ & $64.90$ & $42.33$ & $365.00$ & $\mathbf{1559.15}$ & $\underline{69.29}$ & $\underline{86.76}$ & $\underline{57.91}$ \\
\midrule
\textbf{GW} & $\mathbf{66.43}$ & $\mathbf{62.34}$ & $\mathbf{82.10}$ & $\mathbf{64.05}$ & $\mathbf{67.10}$ & $\mathbf{44.56}$ & $\mathbf{387.86}$ & $\underline{1549.55}$ & $\mathbf{70.84}$ & $\mathbf{87.50}$ & $\mathbf{59.87}$ \\
\midrule
\textcolor{gray}{Optimal} & \textcolor{gray}{$66.43$} & \textcolor{gray}{$62.34$} & \textcolor{gray}{$82.10$} & \textcolor{gray}{$64.05$} & \textcolor{gray}{$67.10$} & \textcolor{gray}{$44.56$} & \textcolor{gray}{$387.86$} & \textcolor{gray}{$1549.55$} & \textcolor{gray}{$70.84$} & \textcolor{gray}{$87.50$} & \textcolor{gray}{$59.87$}\\
\bottomrule
\end{tabular}}
    \caption{Model selection results for Qwen-2.5-7B-Instruct, ``Large" group.}
\label{tab:main_qwen_large}
\end{table*}

\begin{table*}[t]
\centering
\scriptsize
\setlength\tabcolsep{3pt}
\resizebox{\textwidth}{!}{
\begin{tabular}{l|c|cccccccccc}
\toprule
Method & Average & TextVQA & ScienceQA & GQA & AI2D & MMMU & MME-C & MME-P & SEED-I & POPE & RealWorldQA \\
\midrule
Worst & $60.39$ & $50.38$ & $\underline{80.26}$ & $59.37$ & $63.50$ & $42.67$ & $323.57$ & $1350.75$ & $63.77$ & $84.21$ & $51.76$ \\
Accuracy & $\mathbf{64.36}$ & $\mathbf{57.64}$ & $\mathbf{80.52}$ & $\mathbf{62.79}$ & $\mathbf{65.28}$ & $\underline{42.78}$ & $\underline{371.07}$ & $\mathbf{1524.66}$ & $\mathbf{69.03}$ & $\mathbf{86.60}$ & $\mathbf{56.34}$ \\
RSA & $61.31$ & $51.02$ & $79.75$ & $60.45$ & $63.12$ & $\mathbf{43.00}$ & $335.00$ & $1367.21$ & $64.73$ & $84.83$ & $54.12$ \\
CCA & $\underline{63.51}$ & $\underline{54.63}$ & $79.77$ & $\underline{61.35}$ & $\underline{64.51}$ & $41.44$ & $\mathbf{395.36}$ & $\underline{1504.17}$ & $\underline{67.36}$ & $\underline{85.87}$ & $\underline{55.56}$ \\
MutualNN & $61.31$ & $51.02$ & $79.75$ & $60.45$ & $63.12$ & $\mathbf{43.00}$ & $335.00$ & $1367.21$ & $64.73$ & $84.83$ & $54.12$ \\
\midrule
\textbf{GW} & $\mathbf{64.36}$ & $\mathbf{57.64}$ & $\mathbf{80.52}$ & $\mathbf{62.79}$ & $\mathbf{65.28}$ & $\underline{42.78}$ & $\underline{371.07}$ & $\mathbf{1524.66}$ & $\mathbf{69.03}$ & $\mathbf{86.60}$ & $\mathbf{56.34}$ \\
\midrule
\textcolor{gray}{Optimal} & \textcolor{gray}{$64.36$} & \textcolor{gray}{$57.64$} & \textcolor{gray}{$80.52$} & \textcolor{gray}{$62.79$} & \textcolor{gray}{$65.28$} & \textcolor{gray}{$42.78$} & \textcolor{gray}{$371.07$} & \textcolor{gray}{$1524.66$} & \textcolor{gray}{$69.03$} & \textcolor{gray}{$86.60$} & \textcolor{gray}{$56.34$}\\
\bottomrule
\end{tabular}}
    \caption{Model selection results for Qwen-2.5-7B-Instruct, ``Small" group.}
\label{tab:main_qwen_small}
\end{table*}
In previous sections, we have provided a high-level intuitive understanding of why measuring the structural similarity across different modalities can be useful, together with theoretical justifications. In this section, we aim to provide comprehensive experiments to verify our insights from a practical perspective.

\subsection{Experimental setup}

\noindent\textbf{Model pool.} We consider 18 different SOTA vision encoders from diverse sources, summarized in Appendix~\ref{app:exp_model_pool}. In our main experiments, we will divide those vision encoders into two groups, namely \textbf{Small} and \textbf{Large}, and performing model selection in each group separately. The rationale for such practice is that the model selection will only be a challenging issue, if there is no clear gap in models' capability, e.g. size or accuracy. Specifically, vision encoders with less than 500M parameters will be classified as \textbf{Small} group and others will be classified as \textbf{Large}.

\noindent\textbf{Datasets.} We follow the original training recipe of LLaVA-1.5~\citep{liu2024improved}. For the pre-training phase, we use 595K image-text pairs sampled from LAION~\cite{schuhmann2022laion}, Conceptual Caption~\cite{changpinyo2021conceptual} and SBU~\cite{ordonez2011im2text}, that are re-captioned by BLIP~\citep{li2022blip}. For visual instruction tuning phase, we use 665K image-text pairs from the original LLaVA-1.5 paper~\citep{liu2024improved}. To comprehensively evaluate the performances of VLMs in various aspects, from generic question answering, optical character recognition, to general knowledge injection, etc. We adapt 9 popular benchmarks for evaluation purposes, namely TextVQA~\cite{singh2019towards}, ScienceQA~\cite{lu2022learn}, GQA~\cite{hudson2019gqa}, AI2D~\cite{kembhavi2016diagram}, MMMU~\cite{yue2024mmmu}, MME~\cite{fu2025mme}, SEED~\cite{li2023seed}, POPE~\cite{li2023evaluating}, and RealWorldQA~\cite{realworldqa}.

\begin{table*}[t]
\centering
\scriptsize
\setlength\tabcolsep{3pt}
\resizebox{\textwidth}{!}{
\begin{tabular}{l|c|cccccccccc}
\toprule
Method & Average & TextVQA & ScienceQA & GQA & AI2D & MMMU & MME-C & MME-P & SEED-I & POPE & RealWorldQA \\
\midrule
Worst & $57.04$ & $43.46$ & $74.02$ & $59.44$ & $54.37$ & $37.22$ & $267.14$ & $1350.19$ & $63.99$ & $86.64$ & $50.33$ \\
Accuracy & $\underline{60.67}$ & $\underline{53.68}$ & $\underline{78.28}$ & $\mathbf{61.26}$ & $\mathbf{58.16}$ & $\mathbf{39.33}$ & $291.79$ & $\mathbf{1428.32}$ & $\underline{67.68}$ & $\underline{86.98}$ & $\underline{53.46}$ \\
RSA & $58.30$ & $49.20$ & $70.90$ & $\underline{60.11}$ & $56.54$ & $\underline{38.44}$ & $\underline{313.57}$ & $1395.39$ & $65.17$ & $84.67$ & $49.02$ \\
CCA & $\mathbf{61.38}$ & $\mathbf{54.48}$ & $\mathbf{79.39}$ & $58.36$ & $\mathbf{58.16}$ & $38.33$ & $\mathbf{336.43}$ & $\underline{1421.90}$ & $\mathbf{68.81}$ & $\mathbf{87.22}$ & $\mathbf{55.95}$ \\
MutualNN & $57.98$ & $50.18$ & $76.23$ & $58.10$ & $55.12$ & $37.11$ & $307.86$ & $1368.02$ & $64.52$ & $84.94$ & $46.67$ \\
\midrule
\textbf{GW} & $\mathbf{61.38}$ & $\mathbf{54.48}$ & $\mathbf{79.39}$ & $58.36$ & $\mathbf{58.16}$ & $38.33$ & $\mathbf{336.43}$ & $\underline{1421.90}$ & $\mathbf{68.81}$ & $\mathbf{87.22}$ & $\mathbf{55.95}$ \\
\midrule
\textcolor{gray}{Optimal} & \textcolor{gray}{$61.38$} & \textcolor{gray}{$54.48$} & \textcolor{gray}{$79.39$} & \textcolor{gray}{$58.36$} & \textcolor{gray}{$58.16$} & \textcolor{gray}{$38.33$} & \textcolor{gray}{$336.43$} & \textcolor{gray}{$1421.90$} & \textcolor{gray}{$68.81$} & \textcolor{gray}{$87.22$} & \textcolor{gray}{$55.95$}\\
\bottomrule
\end{tabular}}
\caption{Vision encoder selection results for Llama-3.1-8B-Instruct, “Large” group.}
\label{tab:main_llama_large}
\end{table*}

\begin{table*}[t]
\centering
\scriptsize
\setlength\tabcolsep{3pt}
\resizebox{\textwidth}{!}{
\begin{tabular}{l|c|cccccccccc}
\toprule
Method & Average & TextVQA & ScienceQA & GQA & AI2D & MMMU & MME-C & MME-P & SEED-I & POPE & RealWorldQA \\
\midrule
Worst & $54.78$ &$40.29$ & $71.04$ & $52.15$ & $49.97$ & $\underline{37.33}$ & $280.00$ & $1359.32$ & $59.94$ & $85.76$ & $48.37$ \\
Accuracy & $\mathbf{61.31}$ & $\mathbf{56.11}$ & $\mathbf{79.18}$ & $\mathbf{61.86}$ & $\mathbf{58.39}$ & $37.00$ & $\mathbf{301.43}$ & $\underline{1460.43}$ & $\mathbf{67.98}$ & $\mathbf{86.99}$ & $\underline{54.90}$ \\
RSA & $56.19$ & $47.07$ & $73.99$ & $57.52$ & $54.24$ & $36.44$ & $280.00$ & $1281.19$ & $61.07$ & $83.50$ & $49.02$ \\
CCA & $\underline{61.14}$ & $\underline{53.84}$ & $\underline{78.24}$ & $\underline{61.28}$ & $\underline{57.32}$ & $\mathbf{39.78}$ & $\underline{290.36}$ & $\mathbf{1516.23}$ & $\underline{66.06}$ & $\underline{86.79}$ & $\mathbf{55.95}$ \\
MutualNN & $56.19$ & $47.07$ & $73.99$ & $57.52$ & $54.24$ & $36.44$ & $280.00$ & $1281.19$ & $61.07$ & $83.50$ & $49.02$ \\
\midrule
\textbf{GW} & $\mathbf{61.31}$ & $\mathbf{56.11}$ & $\mathbf{79.18}$ & $\mathbf{61.86}$ & $\mathbf{58.39}$ & $37.00$ & $\mathbf{301.43}$ & $\underline{1460.43}$ & $\mathbf{67.98}$ & $\mathbf{86.99}$ & $\underline{54.90}$ \\
\midrule
\textcolor{gray}{Optimal} & \textcolor{gray}{$61.31$} & \textcolor{gray}{$56.11$} & \textcolor{gray}{$79.18$} & \textcolor{gray}{$61.86$} & \textcolor{gray}{$58.39$} & \textcolor{gray}{$37.00$} & \textcolor{gray}{$301.43$} & \textcolor{gray}{$1460.43$} & \textcolor{gray}{$67.98$} & \textcolor{gray}{$86.99$} & \textcolor{gray}{$54.90$}\\
\bottomrule
\end{tabular}}
\caption{Vision encoder selection results for Llama-3.1-8B-Instruct, “Small” group.}
\label{tab:main_llama_small}
\end{table*}

\noindent\textbf{Baselines.} When it comes to selecting vision encoders for VLM training, conventional wisdom usually suggests choosing the model with the highest single-modal performances, which can be mostly represented by the zero-shot \textbf{\textit{Accuracy}}~\citep{lin2024selecting}. Therefore, Accuracy can serve as a naive baseline for vision encoder selection. Also, when it comes to measuring the similarity of heterogeneous spaces, there are also existing heuristics for such purposes, namely representational similarity analysis \textbf{\textit{RSA}}~\cite{kriegeskorte2008representational}, and canonical correlation analysis \textbf{\textit{CCA}}~\cite{morcos2018insights}. Specifically, RSA measures the correlation of pairwise distance matrices from different spaces, which also captures the structural similarity across spaces, however, RSA assumes fixed hard 1-1 correspondence across spaces, making it less flexible comparing to GW distance. As for CCA, it tries to find a pair of linear projections, such that they maximizes the correlation of representations from different space, while this can mitigate the mismatch metric function problem in different space, since CCA unifies them into a joint space, it cannot capture the structural similarity of difference spaces. In addition, \citet{huh2024platonic} also uses the mutual nearest-neighbor metric to measure the level of implicit alignment between visual and textual pre-trained models, in this study, we will also consider it as a baseline and referred as \textbf{\textit{MutualNN}} thereafter. A more complete discussion and implementation details of baselines can be found in the Appendix. Lastly, following \cite{perez2021true}, to showcase the necessity of model selection and the maximum possible gain for employing model selection techniques, we also include the performances of \textbf{\textit{Worst}} model, which is the vision encoder that leads to the worst average performances on 9 benchmarks.

\noindent\textbf{Training configurations.} By default, we adapt the official codebase from LLaVA-NeXT, and follows the training recipe of LLaVA-1.5~\cite{liu2024improved} by a two-stage training process. The first training stage involves pre-training a MLP feature projector, where the learning rate is set to be $2e-3$, batch size is set to be $256$, we use cosine learning rate scheduler and the warm-up ratio is $0.03$. For the second training stage, where we unlocks the parameters of pre-trained vision encoder and LLM for full supervised fine-tuning, at this stage, we set the learning rate to be $2e-5$, batch size to be $128$, cosine learning rate scheduler with warm-up ratio of $0.03$. All VLMs are trained with 8 x GH200. More detailed training configurations can be found in the Appendix~\ref{app:training_config}.

\noindent\textbf{Implementation details.} For the GW estimation, we random sample 1,000 image-text pairs from the alignment dataset. For all vision encoders, we use their \verb|CLS| token from the last layer as the visual feature representation, for all LLM, we use their second-to-last layer encoded hidden representation as textual feature. The default training iterations of the optimal transport optimization is set to be 1,000. For the implementation of GW distance solver, we use the Python Optimal Transport package~\cite{flamary2021pot}.

\subsection{Main Results}
To make sure our results do not rely on specific choice of LLM, we use both \verb|QWen-2.5-7B-Instruct| and  \verb|LLaMa-3.1-8B-Instruct| as base LLM. In each of the following table, we show the ``Optimal" model selected based on the highest average score across 9 benchmarks, highlighting whether any model selection metrics successfully end up with the optimal model, and what is the gap between the selected model and the possible upper-bounds is. For MME benchmark, where the score is not in 0-100 scale, we follow common community practice to divide the MME(Cog.) by 800 and divide MME(Percep.) by 2,000 when calculating the average score~\cite{tong2024cambrian}. For all methods summarized in tables, best performing methods are \textbf{bold} and next best performing methods are \underline{underlined}.

\begin{table}[htbp]
\centering
\begin{tabular}{l|ccc}
\toprule
Method & $|r|$ & $|\rho|$ & R$^2$ \\
\midrule
Accuracy  & $0.4629$ & $0.3934$ & $0.2142$ \\
RSA       & $0.4759$ & $0.4330$ & $0.2265$ \\
CCA       & $0.0430$ & $0.0072$ & $0.0018$ \\
MutualNN  & $0.6081$ & $0.1780$ & $0.3697$ \\
\midrule
\textbf{Ours} & $\mathbf{0.6568}$ & $\mathbf{0.5341}$ & $\mathbf{0.4314}$ \\
\bottomrule
\end{tabular}
\caption{Correlation analysis of different model selection metrics against final VLM performances with Qwen-2.5-7B-Instruct as base model. $|r|$: Absolute Pearson correlation, $\rho$: Absolute Spearman correlation, R$^2$: R-squared coefficient.}
\vspace{-5mm}
\label{tab:correlation}
\end{table}

\noindent\textbf{Qwen-2.5-7B-Instruct.} We summarize the results of using \verb|Qwen-2.5-7B-Instruct| as base model in Table~\ref{tab:main_qwen_large} and Table~\ref{tab:main_qwen_small}. As we can observe from Table~\ref{tab:main_qwen_large}, using GW distance for model selections significantly outperforms other alternatives in nearly all benchmarks, where only in MME (Perception) metrics, our selected model performs slightly worse than best baselines. We note that, for the ``Large" vision encoder category, both zero-shot accuracy and MutualNN favors \verb|DFN5B-ViT-H-378|, whereas GW distance reflects that, \verb|Siglip-SO400M-384|, actually has the highest geometrical similarity to LLM, despite the less favorable zero-shot accuracy score. 

As for the ``Small" group, where results are summarized in Table~\ref{tab:main_qwen_small}, we can see our proposed metrics ties with zero-shot accuracy baseline, while consistently outperforms other baselines. Notably, for zero-shot accuracy baseline, it does not yield as desirable performance as our method in the ``Large" group, suggesting such metrics cannot consistently correlates to final performances across different model group, consequently, if we consider the combined scenerio where all models are in the same model pool, zero-shot accuracy will performs significantly worse than GW.

\noindent\textbf{LLaMa-3.1-8B-Instruct.} We summarize the results of using \verb|LLaMa-3.1-8B-Instruct| as base LLM in Table~\ref{tab:main_llama_large} and Table~\ref{tab:main_llama_small}. Comparing to \verb|Qwen-2.5-7B-Instruct|, the performance of CCA is notably better, ending up selecting the optimal vision encoders in the ``Large" category. In the ``Large" category, both CCA and GW ends up with \verb|siglip-so400m-384|.  However, it fails to identify optimal model in ``Small" group. Also, we should note that, while CCA can select the best or near best vision encoder in these cases, it fails to show sensible correlation to the VLM performances, and performs notably worse in previous cases, suggesting that its observed success is likely to be an outlier here. 

\noindent\textbf{Correlation Analysis.} We present a detailed correlation analysis comparing our proposed method against baseline metrics in Table~\ref{tab:correlation}. Since zero-shot accuracy cannot be straightforwardly computed for certain non-CLIP vision encoders and is not publicly available for \verb|MLCD-ViT-bigG-336|~\cite{an2024multi}, we exclude these models and restrict the analysis to 14 vision encoders. As shown in the table, GW distance is the only metric that achieves both Pearson and Spearman correlations with final VLM performance above $0.5$; moreover, a simple linear fit between GW distance and final VLM performance explains $43\%$ of the variance, substantially outperforming other baselines. Among the baselines, MutualNN attains a reasonably strong Pearson correlation ($>0.5$), but its Spearman and R$^2$ values are considerably weaker. Furthermore, the observed correlation is inversely directed relative to the expected behavior: the correlation between MutualNN score and final VLM performance is negative, whereas a higher MutualNN score should in principle correspond to better VLM performance~\cite{huh2024platonic}. This discrepancy suggests a spurious high correlation rather than a meaningful one.

\subsection{Additional Analysis}

\noindent\textbf{Efficiency Analysis.} Here we present the runtime scaling trend of GW estimation, and the runtime of competitive baseline CCA~\cite{morcos2018insights}. All measured runtime are calculated on a single GH200 GPU/CPU, summarized in Figure~\ref{fig:efficien}. As we can see from the figure, both GW estimation and CCA calculations are pure training-free and inference-only computations, comparing to full model training, which takes approx. 8.5 hours using 8 x GH200 GPUs (68 GPU hours), under default setup (1,000 image-text pairs), GW only takes about ~1 minute to compute. In addition, we can see that the runtime of GW scales linearly as sample size increase, suggesting good scalability with more image-text pairs.

\begin{figure}[t]
  \centering
  \includegraphics[width=\linewidth]{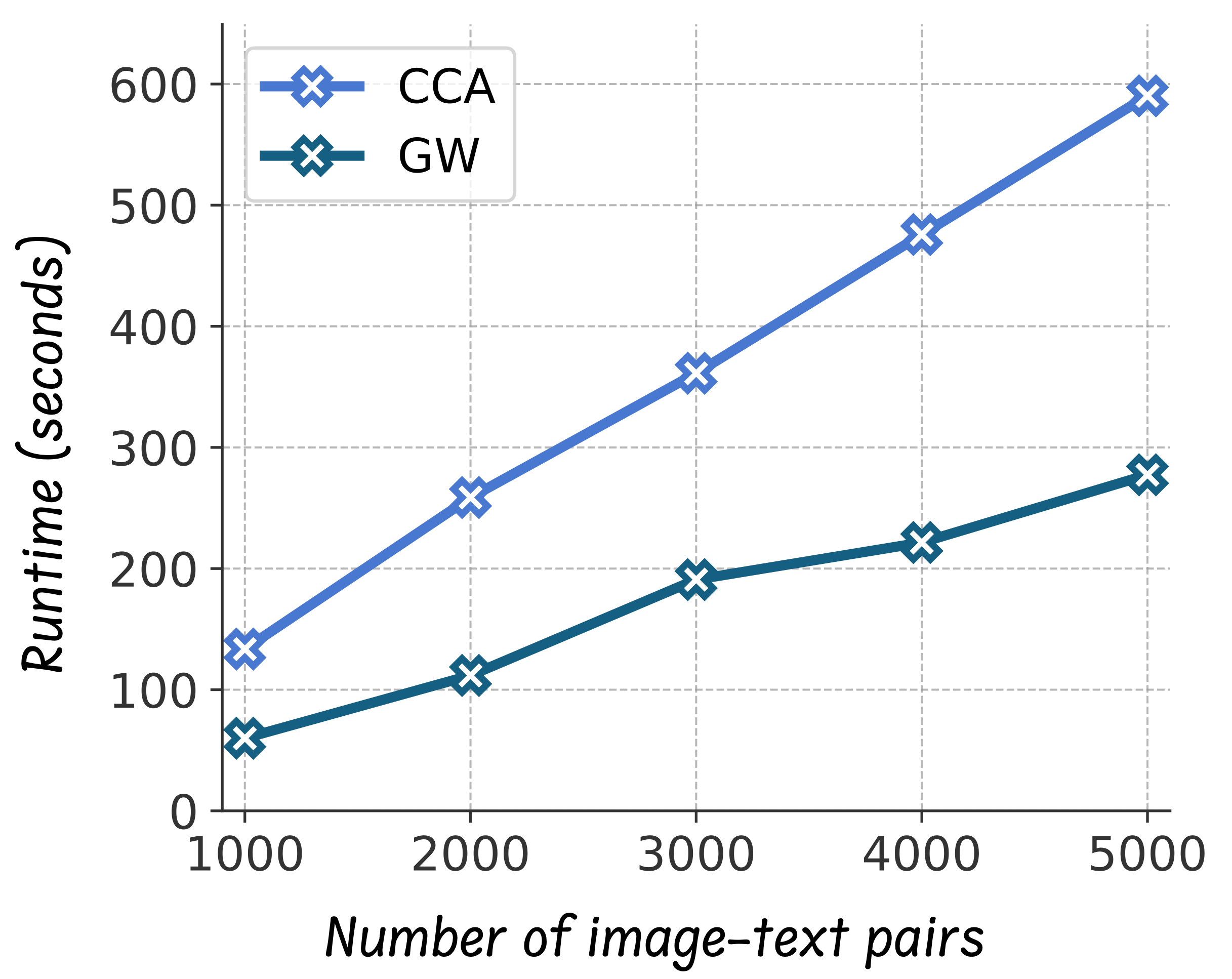}
  \vspace{-5mm}
  \caption{Scaling trend of runtime.}
  \vspace{-5mm}
  \label{fig:efficien}
\end{figure}

\noindent\textbf{Consistency of vision encoder ranking across LLMs.} Another question of interest is how consistent is the ranking of vision encoders for different LLMs, in other words, how generalizable are the choice of vision encoders across different LLMs, if one have spent resources to select the optimal vision encoder for a specific LLM, can such information be transferred to a different LLM without running vision encoder selection again? As shown in Figure~\ref{fig:qwen_vs_llama}, we find that the ranking of the resulting performances from different vision encoders is actually quite consistent across different choice of base LLM, suggesting that, the suitability of vision encoder for VLM training is potentially LLM-agnostic.

\begin{table*}[t]
\centering
\scriptsize
\setlength\tabcolsep{3pt}
\resizebox{\textwidth}{!}{
\begin{tabular}{l|c|cccccccccc}
\toprule
Method & Average & TextVQA & ScienceQA & GQA & AI2D & MMMU & MME-C & MME-P & SEED-I & POPE & RealWorldQA \\
\midrule
Worst & $60.70$ & $52.35$ & $81.44$ & $62.05$ & $63.08$ & $43.22$ & $297.50$ & $1361.82$ & $62.15$ & $85.68$ & $51.76$ \\
Accuracy & $\underline{62.71}$ & $\underline{52.64}$ & $\underline{81.58}$ & $62.05$ & $63.08$ & $\underline{43.78}$ & $308.57$ & $\mathbf{1518.69}$ & $\underline{66.87}$ & $86.49$ & $56.08$ \\
RSA & $\underline{62.71}$ & $\underline{52.64}$ & $\underline{81.58}$ & $62.05$ & $63.08$ & $\underline{43.78}$ & $308.57$ & $\mathbf{1518.69}$ & $\underline{66.87}$ & $86.49$ & $56.08$ \\
CCA & $61.36$ & $46.08$ & $79.70$ & $\underline{62.16}$ & $62.76$ & $40.89$ & $334.29$ & $1389.79$ & $66.24$ & $\underline{87.10}$ & $\underline{57.39}$ \\
MutualNN & $62.27$ & $52.46$ & $80.45$ & $60.77$ & $\underline{65.12}$ & $42.22$ & $\underline{341.07}$ & $1451.55$ & $66.25$ & $85.28$ & $54.90$ \\
\midrule
\textbf{GW} & $\mathbf{65.94}$ & $\mathbf{60.94}$ & $\mathbf{82.29}$ & $\mathbf{63.68}$ & $\mathbf{68.56}$ & $\mathbf{44.00}$ & $\mathbf{371.79}$ & $\underline{1514.35}$ & $\mathbf{70.77}$ & $\mathbf{87.40}$ & $\mathbf{59.61}$ \\
\midrule
\textcolor{gray}{Optimal} & \textcolor{gray}{$65.94$} & \textcolor{gray}{$60.94$} & \textcolor{gray}{$82.29$} & \textcolor{gray}{$63.68$} & \textcolor{gray}{$68.56$} & \textcolor{gray}{$44.00$} & \textcolor{gray}{$371.79$} & \textcolor{gray}{$1514.35$} & \textcolor{gray}{$70.77$} & \textcolor{gray}{$87.40$} & \textcolor{gray}{$59.61$}\\
\bottomrule
\end{tabular}}
    \caption{Model selection results for Qwen-2.5-7B-Instruct, pre-trained on CC3M, ``Large" group.}
\label{tab:main_cc3m_large}
\end{table*}

\begin{table*}[t]
\centering
\scriptsize
\setlength\tabcolsep{3pt}
\resizebox{\textwidth}{!}{
\begin{tabular}{l|c|cccccccccc}
\toprule
Method & Average & TextVQA & ScienceQA & GQA & AI2D & MMMU & MME-C & MME-P & SEED-I & POPE & RealWorldQA \\
\midrule
Worst & $59.90$ & $\underline{49.05}$ & $78.90$ & $58.80$ & $\underline{63.86}$ & $\underline{42.78}$ & $321.07$ & $1350.41$ & $63.15$ & $83.61$ & $51.24$ \\
Accuracy & $\mathbf{65.56}$ & $\mathbf{59.55}$ & $\mathbf{81.44}$ & $\mathbf{63.89}$ & $\mathbf{67.03}$ & $\mathbf{43.78}$ & $\mathbf{381.07}$ & $\mathbf{1521.13}$ & $\mathbf{69.91}$ & $\mathbf{87.73}$ & $\mathbf{58.56}$ \\
RSA & $59.90$ & $\underline{49.05}$ & $78.90$ & $58.80$ & $\underline{63.86}$ & $\underline{42.78}$ & $321.07$ & $1350.41$ & $63.15$ & $83.61$ & $51.24$ \\
CCA & $\underline{61.08}$ & $46.21$ & $\underline{79.72}$ & $\underline{61.98}$ & $63.24$ & $42.33$ & $\underline{322.86}$ & $\underline{1423.73}$ & $\underline{64.58}$ & $\underline{86.97}$ & $\underline{54.25}$ \\
MutualNN & $59.90$ & $\underline{49.05}$ & $78.90$ & $58.80$ & $\underline{63.86}$ & $\underline{42.78}$ & $321.07$ & $1350.41$ & $63.15$ & $83.61$ & $51.24$ \\
\midrule
\textbf{GW} & $\mathbf{65.56}$ & $\mathbf{59.55}$ & $\mathbf{81.44}$ & $\mathbf{63.89}$ & $\mathbf{67.03}$ & $\mathbf{43.78}$ & $\mathbf{381.07}$ & $\mathbf{1521.13}$ & $\mathbf{69.91}$ & $\mathbf{87.73}$ & $\mathbf{58.56}$ \\
\midrule
\textcolor{gray}{Optimal} & \textcolor{gray}{$65.56$} & \textcolor{gray}{$59.55$} & \textcolor{gray}{$81.44$} & \textcolor{gray}{$63.89$} & \textcolor{gray}{$67.03$} & \textcolor{gray}{$43.78$} & \textcolor{gray}{$381.07$} & \textcolor{gray}{$1521.13$} & \textcolor{gray}{$69.91$} & \textcolor{gray}{$87.73$} & \textcolor{gray}{$58.56$}\\
\bottomrule
\end{tabular}}
    \caption{Model selection results for Qwen-2.5-7B-Instruct, pre-trained on CC3M, ``Small" group.}
\label{tab:main_cc3m_small}
\end{table*}

\begin{figure}[t]                 
  \centering
  \includegraphics[width=\linewidth]{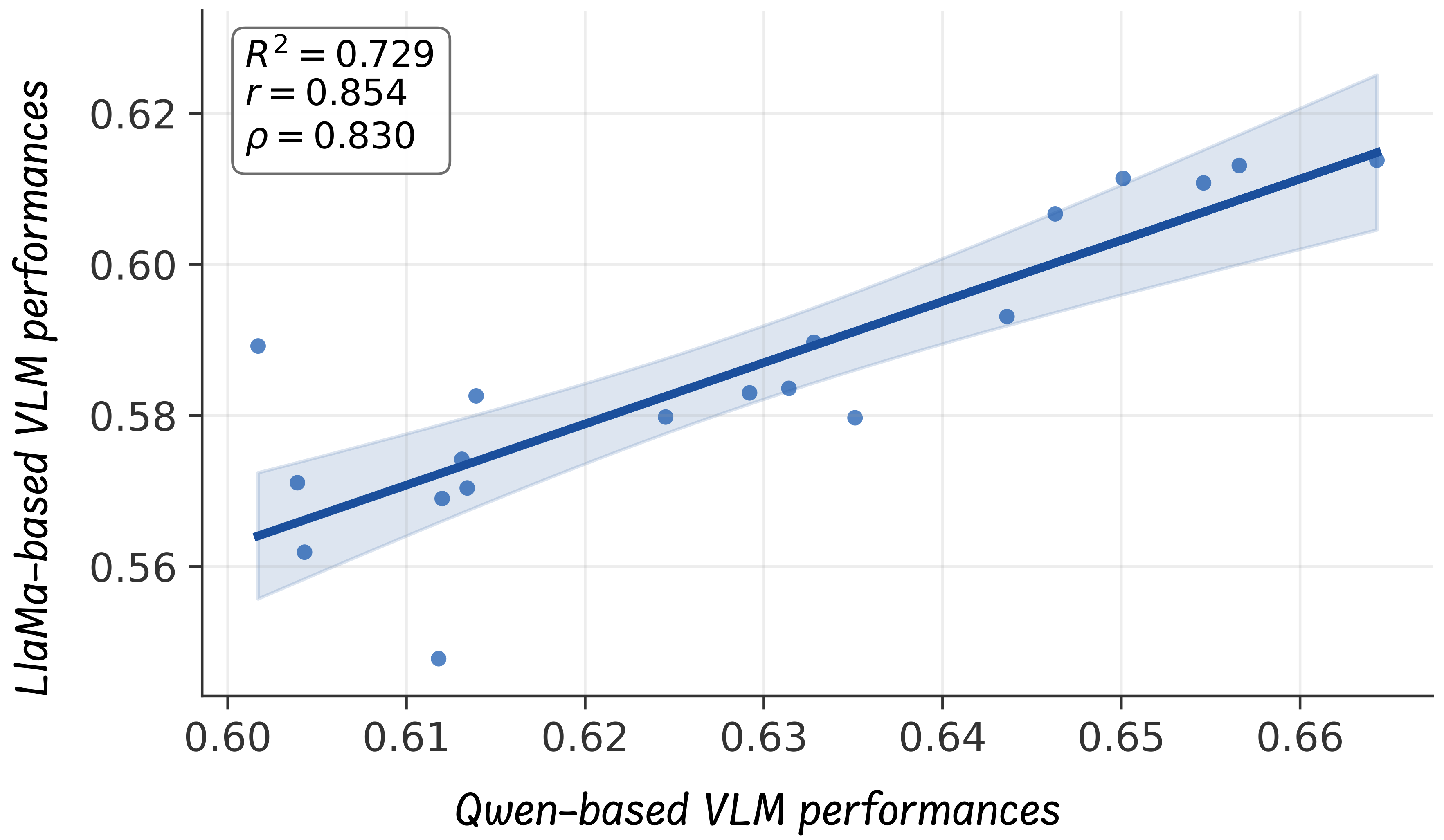}
  \vspace{-5mm}
  \caption{Correlation of the vision encoder ranking between Qwen-2.5-7B-Instruct and LLaMa-3.1-8B-instruct.}
  \label{fig:qwen_vs_llama}
  \vspace{-5mm}
\end{figure}

\noindent\textbf{Changing pre-training datasets.} To rule out the possibility that our obtained results are actually biased to the specific choice of dataset, we change the default pre-training dataset from \verb|LCS-558k| to \verb|CC3M-595k|, and report the performances on \verb|Qwen-2.5-7B-Instruct| on new dataset, summarized in Table~\ref{tab:main_cc3m_large} and Table~\ref{tab:main_cc3m_small}. As we can see from these tables, changing pre-training dataset does not significantly affect the effectiveness of GW distance, verifying that the correlation between GW distance and final VLM performances we observed is not specifically tied to the choice of dataset, rather, has generalizable pattern across different distributions. 

\section{Conclusion}

In this paper, we investigated the question of which vision encoder one should use for VLM training. We first demonstrated that commonly used heuristics, such as model size or model accuracy often fail to identify a suitable vision encoder for VLMs. This motivates a deeper re-examination of which properties of the vision encoder truly matter for successful VLMs. We hypothesize that the \textbf{\textit{compatibility}} between the vision encoder and the LLM plays a crucial yet subtle role in the success of VLMs, and we characterize this compatibility via structural similarity between their representation spaces, measured by the Gromov–Wasserstein (GW) distance. Theoretically, we show that the Lipschitz constant of the optimal cross-modal mapping can be bounded by the GW distance, implying that vision encoders with smaller GW distance to the LLM representations are inherently easier to learn. Empirically, we validate this perspective and show that GW distance serves as an effective, training-free indicator for predicting which vision encoder is suitable for a given LLM.

\section{Limitations}

As our research represents an initial step in this area, the scope of our experiments could be further expanded. For instance, incorporating a broader range of SOTA models would provide more comprehensive validation. Additionally, while the proposed concept is inherently modality-agnostic, our current empirical study focuses exclusively on image and text. Future work could extend this framework to additional modalities, and investigate how cross-modal alignment vary across different modality pairings.

\clearpage
\section*{Acknowledgments}
TLL is partially supported by the following Australian Research Council projects: FT220100318, DP260102466, DP220102121, LP220100527, LP220200949. This research
was supported (in part) by Multidisciplinary Cooperative Research Program in CCS, University of Tsukuba. This work was supported by resources provided by the Pawsey Supercomputing Research Centre’s Setonix Supercomputer (\url{https://doi.org/10.48569/18sb-8s43}), with funding from the Australian Government and the Government of Western Australia.

{
    \small
    \bibliographystyle{ieeenat_fullname}
    \bibliography{main}
}

\clearpage
\setcounter{page}{1}
\maketitlesupplementary

\section{Theoretical Analysis (Full)} \label{app:theory}

Due to space limitations, we provide some technical analysis and the full proof of Theorem 1 here. 

\subsection{Definitions}\label{app:theory:def}

Note the following definition is given in \cite{lu2023theory}, which we put here for completeness. 
\begin{definition}[Approximate realizability\cite{lu2023theory}]
    Define the `approximate realizability' of a hypothesis class $\mathcal{G}$ on a paired dataset $D = \{(\bm{x}_1,\bm{y}_1),\dots,(\bm{x}_n,\bm{y}_n)\}$ as
    \begin{align}
        \mathcal{R}(\mathcal{G},D) = \min_{g\in\mathcal{G}}\frac{1}{n}\sum_{i=1}^{n}\norm{g(\bm{x}) - \bm{y}}.
    \end{align}
\end{definition}

\subsection{Lemma 1}\label{app:theory:proof_lemma1}

\begin{lemma}\label{lemma:GW}
    For any $(\bm{x},\bm{y})$, $(\bm{x}',\bm{y}') \in \supp(\pi^*_{\infty})$, we have:
    \begin{align}
        \left|d_{\mathcal{Y}}(g^*(\bm{x}), g^*(\bm{x}')) - d_{\mathcal{\mathcal{X}}}(\bm{x},\bm{x}')\right| \leq \text{GW}_{\infty} + 2\rho^{*}_{\pi_{\infty}}.
    \end{align}
\end{lemma}

\begin{proof}
To prove the lemma, we need to prove \[d_{\mathcal{Y}}(g^*(\bm{x}), g^*(\bm{x}')) - d_{\mathcal{\mathcal{X}}}(\bm{x},\bm{x}') \leq \text{GW}_{\infty} + 2\rho^{*}_{\pi_{\infty}}, \quad (\text{case 1})\] \[d_{\mathcal{\mathcal{X}}}(\bm{x},\bm{x}') - d_{\mathcal{Y}}(g^*(\bm{x}), g^*(\bm{x}')) \leq \text{GW}_{\infty} + 2\rho^{*}_{\pi_{\infty}}, \quad (\text{case 2})\] respectively.

\noindent\textbf{Case 1.} 
We start by bounding $d_{\mathcal{Y}}(g^*(\bm{x}), g^*(\bm{x}'))$. By triangle inequality, we have:
\begin{align*}
    d_{\mathcal{Y}}(g^*(\bm{x}), g^*(\bm{x}')) \leq d_{\mathcal{Y}}(g^*(\bm{x}),\bm{y}) + d_{\mathcal{Y}}(\bm{y},g^*(\bm{x}')).
\end{align*}
Applying triangle inequality again on $d_{\mathcal{Y}}(g^*(\bm{x}'),\bm{y})$:
\begin{align*}
    d_{\mathcal{Y}}(g^*(\bm{x}'),\bm{y}) \leq d_{\mathcal{Y}}(\bm{y},\bm{y}') + d_{\mathcal{Y}}(g^*(\bm{x}'),\bm{y}').
\end{align*}
Putting them back together: 
\begin{align*}
    d_{\mathcal{Y}}(g^*(\bm{x})&, g^*(\bm{x}')) \leq \\ &d_{\mathcal{Y}}(g^*(\bm{x}),\bm{y}) + d_{\mathcal{Y}}(\bm{y},\bm{y}') + d_{\mathcal{Y}}(g^*(\bm{x}'),\bm{y}').
\end{align*}
Subtracting $d_{\mathcal{\mathcal{X}}}(\bm{x},\bm{x}')$ from both sides, we have: 
\begin{align*}
    &d_{\mathcal{Y}}(g^*(\bm{x}), g^*(\bm{x}')) - d_{\mathcal{\mathcal{X}}}(\bm{x},\bm{x}') \\&\leq d_{\mathcal{Y}}(g^*(\bm{x}),\bm{y}) + d_{\mathcal{Y}}(\bm{y},\bm{y}') + d_{\mathcal{Y}}(g^*(\bm{x}'),\bm{y}') - d_{\mathcal{\mathcal{X}}}(\bm{x},\bm{x}').
\end{align*}
Let $\pi^{*}_{\infty}$ be the underlying optimal coupling for the $\text{GW}_{\infty}$, then by Definition \ref{def:GW_inf} we can substitute the $\text{GW}_{\infty}$ into the inequality and obtain:
\begin{align*}
d_{\mathcal{Y}}(g^*(\bm{x})&, g^*(\bm{x}')) - d_{\mathcal{\mathcal{X}}}(\bm{x},\bm{x}') \\
     &\leq \text{GW}_{\infty} + d_{\mathcal{Y}}(g^*(\bm{x}),\bm{y}) + d_{\mathcal{Y}}(g^*(\bm{x}'),\bm{y}')\\
     &\leq \text{GW}_{\infty} + 2\rho^{*}_{\pi_{\infty}}, \quad \text{(Definition \ref{def:bayes})}
\end{align*}
which proves case 1.

\noindent\textbf{Case 2.} Proof of case 2 is similar to case 1, we start by applying triangle inequality on $d_{\mathcal{Y}}(\bm{y},\bm{y}')$ and obtain:
\begin{align*}
    d_{\mathcal{Y}}(\bm{y},\bm{y}') \leq d_{\mathcal{Y}}(\bm{y},g^*(\bm{x})) + d_{\mathcal{Y}}(g^*(\bm{x}),\bm{y}'). 
\end{align*}
Applying triangle inequality again on $d_{\mathcal{Y}}(g^*(\bm{x}),\bm{y}')$: 
\begin{align*}
    d_{\mathcal{Y}}(g^*(\bm{x}),\bm{y}') \leq d_{\mathcal{Y}}(g^*(\bm{x}), g^*(\bm{x}')) + d_{\mathcal{Y}}(g^*(\bm{x}'),\bm{y}'). 
\end{align*}
Putting them back together: 
\begin{align*}
    d_{\mathcal{Y}}(\bm{y},&\bm{y}') \leq \\ &d_{\mathcal{Y}}(\bm{y},g^*(\bm{x})) + d_{\mathcal{Y}}(g^*(\bm{x}), g^*(\bm{x}')) + d_{\mathcal{Y}}(g^*(\bm{x}'),\bm{y}').
\end{align*}
Apply Definition \ref{def:GW_inf} similarly:
\begin{align*}
    &d_\mathcal{\mathcal{X}}(\bm{x},\bm{x}') \le \text{GW}_\infty + d_{\mathcal{Y}}(\bm{y},g^*(\bm{x})) \\ & \qquad \qquad \qquad + d_{\mathcal{Y}}(g^*(\bm{x}), g^*(\bm{x}')) + d_{\mathcal{Y}}(g^*(\bm{x}'),\bm{y}'). 
\end{align*}
Subtracting $d_{\mathcal{Y}}(g^*(\bm{x}), g^*(\bm{x}'))$ from both sides, we have 
\begin{align*}
    d_{\mathcal{\mathcal{X}}}(\bm{x},\bm{x}')& - d_{\mathcal{Y}}(g^*(\bm{x}), g^*(\bm{x}')) \\ &\leq d_{\mathcal{Y}}(\bm{y},g^*(\bm{x})) + d_{\mathcal{Y}}(g^*(\bm{x}'),\bm{y}') + \text{GW}_{\infty}, \\
     &\leq \text{GW}_{\infty} + 2\rho^{*}_{\pi_{\infty}},
\end{align*}
which completes the proof. 
\end{proof}

\subsection{Proof of Theorem 1}\label{app:theory:proof_theo1}
\begin{proof}
In Lemma \ref{lemma:GW}, we have shown that the deviation of pairwise distance, caused by a mapping function $g^*$ from space $\mathcal{\mathcal{X}}$ to space $\mathcal{Y}$ can be bounded by the $\infty$-norm GW distance and the supremum of the mapping error of $g^{*}_{\pi}$. 
This can be intuitively translated to a bound on the Lipschitz-constant of the function. 
That is, by Lemma 1 we have 
\begin{align*}
     \frac{d_{\mathcal{Y}}(g^*(\bm{x}), g^{*}(\bm{x}'))}{d_{\mathcal{X}}(\bm{x},\bm{x}')} \leq \frac{2\rho^{*}_{\pi_{\infty}}+\text{GW}_{\infty}}{d_{\mathcal{X}}(\bm{x},\bm{x}')} + 1.
\end{align*}
By the definition of Lipschitz-continuity, $\sup_{x\neq x'}\frac{d_{\mathcal{Y}}(g^*(\bm{x}), g^*(\bm{x}'))}{d_{\mathcal{X}}(\bm{x},\bm{x}')}$ is exactly the Lipschitz constant of $g^*$. Taking the supremum over all pairs $x \ne x'$ on both sides and noting that $d_\mathcal{X}(\bm{x},\bm{x}') \ge r$ 
yields the upper bound on $L_{g^*}$:
\begin{align*}
     L_{g^*} &=\sup_{x\ne x'} \frac{d_{\mathcal{Y}}(g^*(\bm{x}), g^*(\bm{x}'))}{d_{\mathcal{X}}(\bm{x},\bm{x}')} \\ &\leq \sup_{x\ne x'} \left(1 + \frac{2\rho^{*}_{\pi_{\infty}}+\text{GW}_{\infty}}{d_{\mathcal{X}}(\bm{x},\bm{x}')}\right) \\ &= 1 + \frac{2\rho^{*}_{\pi_{\infty}}+\text{GW}_{\infty}}{r}.
\end{align*}
\end{proof}

\section{Related Works}
\label{app:relate}

\subsection{Performance Prediction for Model Selection.} Due to the prohibitive cost for full-training every single candidate large models, performance prediction, or transferability estimation, has become an increasing attentive area. However, most of the works in this domain focus on classification tasks only~\cite{kim2016learning,tran2019transferability,nguyen2020leep,bolya2021scalable,you2021logme,shao2022not,ding2022pactran,zhang2023model,tu2024ranked}, and are hard to naively migrated to generative tasks. In addition, the problem being studied in this paper differs fundamentally from existing works in transferability estimation in the following sense: existing works mostly consider a \textit{task-centric} perspective, that is, given the joint distribution of a target task, how do we evaluate the fitness of the model without actually fine-tuning them. Whereas the scope of this study is focus on the interplay of multi-modal representations, and how the choice of pre-train models influence it. Another recent work~\cite{zhang2025assessing} considers the similar notion of assessing the quality of vision encoder by examining their clustering quality.

\subsection{The Platonic Representation Hypothesis (full)}

However, it is crucial to note that MutualNN has several drawbacks compared to GW, particularly in the context of model selection: (i) it suffers from sensitivity to an additional hyper-parameter (top-$k$ neighbors). Since the task is training-free model selection, it is not feasible to optimize this hyper-parameter in hindsight based on post-VLM performance; (ii) MutualNN does not enjoy the theoretical properties of GW distance. While we can intuitively understand that MutualNN ``approximates" the geometrical similarity between spaces, how to translate that into a formal guarantee, and how that influences the learnability of mappings across spaces remains dubious; (iii) akin to RSA, MutualNN enforces a ``hard" correspondence between known image-text pairs. While such an assumption is reasonable (or even superior) when semantic information overlaps perfectly, it fails when modalities are complementary. If the semantic content of an image and text pair is distinct but related (non-overlapping), MutualNN incorrectly assumes the corresponding objects are affine and fails to capture the relationship. We hypothesize this is exactly why the original authors chose a relatively small $k$: they implicitly assume that for the most semantically aligned sub-regions, information overlaps rather than complements. In contrast, because GW learns the correspondence via the coupling matrix, it is inherently more robust to such complementary scenarios (as such correspondence will be down-weighted).

\section{Experimental Details}\label{app:exp_details}

We summarize the key training configurations of pre-training feature alignment (stage 1) and visual instruction tuning (stage 2) in Table~\ref{tab:train_config}.

\subsection{Training Configurations}\label{app:training_config}

\begin{table}[ht]
  \centering
  \setlength{\tabcolsep}{9pt}   
  \renewcommand{\arraystretch}{1.12} 
  \begin{tabular}{l|cc}
    \toprule
    \textbf{Configuration} & Stage-1 & Stage-2 \\
    \midrule
    learning rate & $2e-3$ & $2e-5$  \\
    learning scheduler  & cosine & cosine    \\
    warmup ratio & 0.03 & 0.03   \\
    global batch size &   256   & 128        \\
    training epoch           & 1 & 1     \\
    max sequence length      & 2048 & 2048      \\
    \bottomrule
  \end{tabular}
  \vspace{3mm}
  \caption{LLaVa-1.5 training configuration.}
  \label{tab:train_config}

\end{table}

\subsection{Details of Model Pool}\label{app:exp_model_pool}

\begin{table}[ht]
  \centering
  \setlength{\tabcolsep}{4pt}
  \renewcommand{\arraystretch}{1.1}

  \resizebox{\linewidth}{!}{%
    \begin{tabular}{lllllc}
      \toprule
    Provider & Source & Architecture & Size & ImgNet-1k\\
    \midrule
    \rowcolor{gray!15} \multicolumn{5}{l}{\textit{``Large" group: vision encoders size larger than ViT-L}} \\
    Laion & OpenCLIP~\cite{cherti2022reproducible} & \code{ViT-bigG-14} & 2540M &  80.09 \\
    BAAI & MLCD~\cite{an2024multi} & \code{ViT-bigG-14} & 1842M & - \\
    Laion & OpenCLIP~\cite{cherti2022reproducible} & \code{ViT-g-14} & 1367M &  78.47 \\
    Meta & DINO-v2~\cite{oquab2023dinov2} & \code{giant} & 1136M & - \\
    Apple & DFN~\cite{fang2023data} & \code{ViT-H-14} & 987M & 84.37 \\
    Meta & CLIP~\cite{radford2021learning} & \code{ViT-H-14} & 986M & 80.51 \\
    Laion & OpenCLIP~\cite{cherti2022reproducible} & \code{ViT-H-14} & 986M &  77.96 \\
    Google & SigLIP~\cite{zhai2023sigmoid} & \code{SoViT-400m} & 877M &  83.08 \\
    \midrule
    \rowcolor{gray!15} \multicolumn{5}{l}{\textit{``Small" group: vision encoders size smaller or equal to ViT-L}} \\
    Google & SigLIP~\cite{zhai2023sigmoid} & \code{ViT-L-16} & 652M &  82.07 \\
    Apple & DFN~\cite{fang2023data} & \code{ViT-L-14} & 428M  & 81.41 \\
    Laion & OpenCLIP~\cite{cherti2022reproducible} & \code{ViT-L-14} & 428M & 79.21 \\
    OpenAI & CLIP~\cite{radford2021learning} & \code{ViT-L-14} & 428M & 76.56 \\
    Laion & OpenCLIP~\cite{cherti2022reproducible} & \code{ViT-L-14} & 428M & 75.25 \\
    Meta & DINO-v2 & \code{Large} & 304M &  - \\
    Google & SigLIP~\cite{zhai2023sigmoid} & \code{ViT-B-16} & 203M & 78.49 \\
    OpenAI & CLIP~\cite{radford2021learning} & \code{ViT-B-16} & 150M & 68.34 \\
    Meta & MetaCLIP~\cite{xu2023demystifying} & \code{ViT-B-16} & 149M &  72.12 \\
    Meta & DINO-v2~\cite{oquab2023dinov2} & \code{base} & 87M &  - \\

      \bottomrule
    \end{tabular}%
  }
  \cprotect\caption{Collection of vision-encoders in our experiments.}
  \label{tab:model-summary}
\end{table}

We summarized all vision encoders used in this study in Table~\ref{tab:model-summary}, within each group, vision encoders are ordered by their size, models with similar size are ordered by accuracy. We try to make the collection as diverse as possible, covering models from different providers and different architectures. Any entry shown as ``-" in Table~\ref{tab:model-summary} means zero-shot accuracy cannot be naively computed (DINO family), or there is no open reported zero-shot accuracy and paired text encoder is not found (MLCD).

\subsection{Implementation Details of Baselines}

\paragraph{RSA.} For the implementation of Representational Similarity Analysis (RSA)~\cite{kriegeskorte2008representational}, we use 1,000 randomly sampled image-text pairs to estimate the within-space pairwise similarity matrices. We employ cosine similarity as the comparison metric.

\paragraph{CCA.} We implement Canonical Correlation Analysis (CCA)~\cite{morcos2018insights} using scikit-learn~\cite{scikit-learn}. We increase the sample size to 5,000 pairs to ensure it exceeds the maximum feature dimension of either modality. We set the number of components to 10, as higher values yielded no significant changes in model ranking. The optimization is set to a maximum of 500 iterations with a tolerance of $10^{-6}$, following scikit-learn defaults.

\paragraph{MutualNN.} For MutualNN~\cite{huh2024platonic}, we use 1,000 randomly sampled pairs and set the number of neighbors to $k=10$, following the official implementation. To align with our GW distance setup, we treat the image modality as the source domain, calculating overlap based on the nearest neighbors with respect to the image features.

\end{document}